\documentclass[nohyperref]{article}

\usepackage{microtype}
\usepackage{graphicx}
\usepackage{subfigure}
\usepackage{booktabs} 

\usepackage{hyperref}

\usepackage[accepted]{icml2022_TAGML}

\usepackage{amsmath}
\usepackage{amssymb}
\usepackage{mathtools}
\usepackage{amsthm}

\makeatletter
\newtheorem*{rep@theorem}{\rep@title}
\newcommand{\newreptheorem}[2]{%
\newenvironment{rep#1}[1]{%
 \def\rep@title{#2 \ref{##1}}%
 \begin{rep@theorem}}%
 {\end{rep@theorem}}}
\makeatother

\usepackage{tikz-cd}

\usepackage[capitalize,noabbrev]{cleveref}

\theoremstyle{plain}
\newtheorem{theorem}{Theorem}[section]
\newtheorem{proposition}[theorem]{Proposition}
\newtheorem{lemma}[theorem]{Lemma}
\newtheorem{corollary}[theorem]{Corollary}
\theoremstyle{definition}
\newtheorem{definition}[theorem]{Definition}

\theoremstyle{remark}
\newtheorem{remark}[theorem]{Remark}
\newreptheorem{theorem}{Theorem}
\newreptheorem{lemma}{Lemma}
\newreptheorem{corollary}{Corollary}

\usepackage[textsize=tiny]{todonotes}

\icmltitlerunning{A Topological characterisation of Weisfeiler-Leman equivalence classes}

\begin{document}

\twocolumn[
\icmltitle{A Topological characterisation of Weisfeiler-Leman equivalence classes}




\icmlsetsymbol{equal}{*}

\begin{icmlauthorlist}
\icmlauthor{Jacob Bamberger}{yyy}

\end{icmlauthorlist}

\icmlaffiliation{yyy}{Signal Processing Laboratory 2, EPFL, Lausanne, Switzerland}

\icmlcorrespondingauthor{Jacob Bamberger}{jacob.bamberger@epfl.ch}

\icmlkeywords{Graph Neural Network, Geometric Deep Learning, Weisfeiler-Leman, TAG in ML}

\vskip 0.3in
]



\printAffiliationsAndNotice{}  

\begin{abstract}
Graph Neural Networks (GNNs) are learning models aimed at processing graphs and signals on graphs. The most popular and successful GNNs are based on message passing schemes. Such schemes inherently have limited expressive power when it comes to distinguishing two non-isomorphic graphs. In this article, we rely on the theory of covering spaces to fully characterize the classes of graphs that GNNs cannot distinguish. We then generate arbitrarily many non-isomorphic graphs that cannot be distinguished by GNNs, leading to the GraphCovers dataset. We also show that the number of indistinguishable graphs in our dataset grows super-exponentially with the number of nodes. Finally, we test the GraphCovers dataset on several GNN architectures, showing that none of them can distinguish any two graphs it contains.
\end{abstract}

\section{Introduction}
\label{sec:introduction} Many modern Graph Neural Networks (GNNs) use a message passing scheme \cite{pmlr-v70-gilmer17a}.
Recent studies compared such message passing GNNs to the Weisfeiler-Leman (WL) graph isomorphism test \yrcite{weisfeilerleman}, and used this analogy to show that such GNNs have limited expressive power \cite{Morris_Ritzert_Fey_Hamilton_Lenssen_Rattan_Grohe_2019, GIN}.
The challenge of finding architectures that overcome these limits is a very active research question, and has inspired much work, e.g. \cite{Morris_Ritzert_Fey_Hamilton_Lenssen_Rattan_Grohe_2019, 10.5555/3524938.3525258, papp2021dropgnn, bevilacqua2022equivariant}.

In this paper, we build upon the WL analogy to fully characterize graphs that GNNs cannot distinguish. We do so by extending Leighton's Theorem \yrcite{LEIGHTON1982231} to give a topological description of graphs that cannot be distinguished by the WL test. The description uses covering maps, a notion of local isomorphism between spaces.
Our proofs being mainly constructive, we use them to generate a family of graphs that are non-isomorphic but indistinguishable by WL, leading to the GraphCovers dataset. We then show that the number of generated isomorphism classes scales super-exponentially with the desired number of nodes. Finally, we benchmark the GraphCovers dataset on the graph classification task with several common GNN architectures, showing that none of them can distinguish any two graphs it contains.

Before introducing the terms and concepts, we state the main ``battle horse'' of this paper, which is a rephrasing of Leighton's Theorem. See Figure~\ref{fig:coveringexample} for an illustration.

\begin{theorem}\label{thm:main} \cite{LEIGHTON1982231}
Two connected graphs $G$ and $G'$ are indistinguishable by the WL test if and only if they have the same order and have isomorphic universal covers, or equivalently, if and only if they have the same order and have isomorphic finite covers.
\end{theorem}

\begin{figure}[ht]
    \vskip 0.2in
    \begin{center}
    \centerline{\includegraphics[width=\columnwidth]{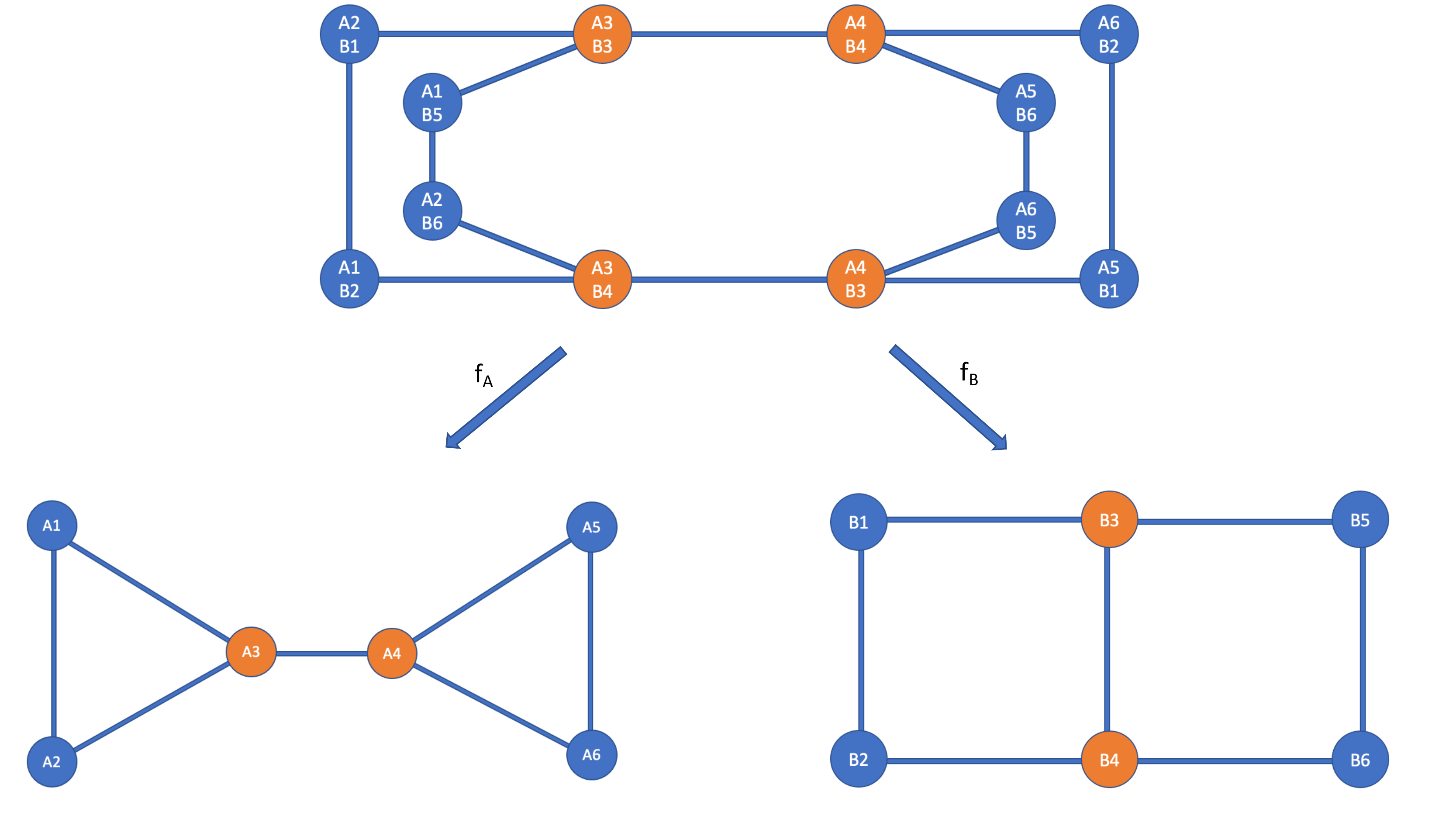}}
    \caption{Two graphs $G_A$ (left) and $G_B$ (right) and their stable WL colourings which are indistinguishable, and a common covering space (top), with maps $f_A$ and $f_B$ defined by the vertex labels.}
    \label{fig:coveringexample}
    \end{center}
    \vskip -0.2in
\end{figure}

We give a self contained proof of a simpler version of Theorem~\ref{thm:main}, namely that $G$ and $G'$ are indistinguishable by WL if and only if they have the same order and have isomorphic universal covers. The full theorem is not proven for two reasons: first, the partial result will be sufficiently strong for us to generate the GraphCovers dataset; second, we expect our proofs to generalize to simplicial and cellular message passing as defined in \cite{NEURIPS2021_157792e4, bodnar2021weisfeiler}, whereas it is known that Leighton's Theorem does not generalize to arbitrary simplicial complexes, e.g. \cite{10.5555/924333}.


To fully characterize WL equivalence classes, we need to also consider disconnected graphs, since for example the WL test is fooled by the length $6$ cycle and the disjoint union of two length $3$ cycles. Hence we generalize Theorem~\ref{thm:main} to get the following classification result for WL equivalence classes. The proof can be found in the Appendix~\ref{app:maincor}.
\begin{theorem}\label{cor:disconnected}
    Two graphs $G$ and $G'$ of bounded degree are indistinguishable by the WL test if and only if they admit decompositions into disjoint unions of connected graphs:
    \[ G = \bigoplus_{i\in \mathcal{I}} \bigoplus_{j\in\mathcal{J}_i} G_{ij} \text{ and } G' = \bigoplus_{i\in\mathcal{I}}\bigoplus_{j'\in\mathcal{J}_{i}'} G'_{ij'} \ ,\] such that the following holds
    \begin{itemize}
    \vspace{-0.15in}
        \item for every $i\in\mathcal{I}$,  $|\bigoplus_{j\in\mathcal{J}_i} G_{ij}| =  |\bigoplus_{j'\in\mathcal{J}'_{i}} G'_{ij'}|$
        \vspace{-0.075in}
        \item for every $i\in\mathcal{I}$, there is a graph $H_i$ that covers both $G_{ij}$ and $G'_{ij'}$ for all $j\in\mathcal{J}_i$ and $j'\in\mathcal{J}'_i$.
    \end{itemize}
    
    Moreover, if $|\bigoplus_{j\in\mathcal{J}_i} G_{ij}|<\infty$, the cover $H_i$ can be chosen to be finite.
\end{theorem}

\subsection{Related work}
\textbf{Coverings: } Covering maps for graphs are a special case of covering maps for topological spaces, see Chapter $1$ in \cite{Hatcher:478079} for a classical introduction to coverings, and Chapter~$1$.A for the case of graphs. Covering graphs have been fruitful in many ways, both in applications and in pure mathematics: they were first applied to problems in multi-processor networks \cite{angluin}, since then they have been at the core of the field of geometric and combinatorial group theory, e.g. \cite{stallings}. They have also been studied from a purely graph theoretic perspective, e.g. \cite{GROSS1977273}, from a probabilistic viewpoint \cite{randomcovers2002}, and from a spectral one \cite{ChungFan1998CHKa}. However, we are not aware of applications of coverings to GNNs.\\
\textbf{GNN expressiveness:} A successful research topic has been to understand how good GNNs are at distinguishing non-isomorphic graphs. To do so, GNNs have been compared to the WL test, an algorithm capable of distinguishing most, but not all, non-isomorphic graphs \cite{weisfeilerleman}. Using the WL analogy, researchers have shown that GNNs have limited expressive power when it comes to distinguishing non-isomorphic graphs \cite{Morris_Ritzert_Fey_Hamilton_Lenssen_Rattan_Grohe_2019, GIN}. These results have triggered a body of work that study the combinatorial structures GNNs cannot distinguish, e.g. \cite{Loukas2020What, NEURIPS2020_23685a24, NEURIPS2020_75877cb7, 10.5555/3524938.3525258}. A line of work showed that GNN architectures gain in expressive power by encoding node ids either uniquely or randomly \cite{Loukas2020What, ijcai2021-291, doi:10.1137/1.9781611976700.38}, or through the use of port numbers - an idea adapted from the theory of distributed local algorithms \cite{NEURIPS2019_635440af}. Another approach to gaining expressive power is to extend the WL analogy to the higher order $k$-WL tests \cite{immerman1990describing}. This idea lead to several higher order architectures which are provably more expressive \cite{Morris_Ritzert_Fey_Hamilton_Lenssen_Rattan_Grohe_2019, NEURIPS2019_bb04af0f}. Another approach is to encode substructures by treating graphs as a (possibly random) set of subgraphs \cite{papp2021dropgnn, bevilacqua2022equivariant}, or by treating them as simplicial or cellular complexes \cite{NEURIPS2021_157792e4, bodnar2021weisfeiler}.

\section{Proof of the main Theorem}
\label{sec:proofofmain}
\subsection{Definitions}
\label{subsec:definitions}

We start by defining graphs, (iso)morphisms and coverings between graphs, and (iso)morphisms between coverings.
\begin{definition}\label{def:graphs}
An undirected graph $G$ consists of a countable vertex set $\mathcal{V}(G)$ and an edge set $\mathcal{E}(G)\subseteq \mathcal{V}(G)\times \mathcal{V}(G)$ that is symmetric (i.e. $vw\in \mathcal{E}(G) \implies wv\in \mathcal{E}(G)$).
When the underlying graph is clear from context, we refer to them as $\mathcal{V}$ and $\mathcal{E}$.
The neighbors of a vertex $v$ in a graph $G$ is defined as $\mathcal{N}_G(v):=\{w\in \mathcal{V}(G)\ |\ vw\in \mathcal{E}(G)\}$. For each vertex $v$, we define it's neighborhood as the subgraph $G_v \subseteq G$ with vertices $\mathcal{V}(G_v) := \{v\}\bigcup \mathcal{N}_G(v)$ and edges being the edges $uv$ and $vu$ connecting $v$ to a neighbor $u\in \mathcal{N}_{v}$. The order of $G$ is $|\mathcal{V}(G)|$.\end{definition}

\begin{definition}
A morphism $\phi: H \rightarrow G$ between two graphs $H$,  $G$ is a map (denoted by the same symbol) $\phi: \mathcal{V}(H) \rightarrow \mathcal{V}(G)$ that preserves edges, i.e. $vw\in \mathcal{E}(H) \implies \phi(v)\phi(w)\in \mathcal{E}(G)$.
It is an isomorphism if $\phi$ is bijective on vertices and if the induced map on edges is also bijective, i.e. $vw\in \mathcal{E}(H) \iff \phi(v)\phi(w)\in \mathcal{E}(G)$.
It is a covering if it is surjective on vertices, and if $\phi$ restricts to an isomorphism on all neighbor graphs $G_v$. That is, if $\forall v \in \mathcal{V}(G), \ \forall v' \in \phi^{-1}(v)$, and $\forall w'\in N_G(v'), \ \exists! v'w'\in \mathcal{E}(H)$ such that $\phi(v)\phi(w) = v'w'$. If $\phi$ is a covering map, we refer to the co-domain as the base graph.
\end{definition}

\begin{remark}
Note that isomorphisms are a special case of coverings. See Figure~\ref{fig:coveringexample} for examples of covering maps. 
\end{remark}

\begin{definition}\label{def:coveringmorphism}
A morphism between covering spaces $\pi: H\rightarrow G$ and $\pi': H'\rightarrow G$ of a same base graph $G$ is a graph morphism $\phi: H \rightarrow H'$, such that $\pi'\circ\phi = \pi$, i.e. such that the following diagram commutes:
\begin{center}
\begin{tikzcd}
H \arrow[rr, "\phi"] \arrow[rd, "\pi"'] &   & H' \arrow[ld, "\pi'"] \\
                                        & G &                      
\end{tikzcd}
\end{center}
It is an isomorphism of covering space if $\phi$ is an isomorphism of graphs. 
\end{definition}

\begin{remark}\label{rem:finiteness}
For any graph $G$ considered, we assume that it has bounded degree (i.e. $ \exists K\in \mathbb{N}$ such that $\forall v \in \mathcal{V}(G)$, $|\mathcal{N}(v)|<K$). Note that both finite graphs, and coverings of finite graphs satisfy this property. We refer to undirected graphs with this property simply as graphs.
\end{remark}

We now introduce the notion of colorings and neighborhood colorings, which are essential to define the color refinement, itself used to define the WL test.

\begin{algorithm}[tb]
\caption{Color refinement}
\label{alg:colorrefinement}
\begin{algorithmic}
   \STATE {\bfseries Input:} Graph $G$\\
   \STATE {Fix $c_0$ a constant coloring of $G$, set $i=0$}
   \STATE {Fix an injective $HASH: CN(G)\rightarrow\mathbb{N}$}
   \WHILE{$c_i \ \& \ c_{i-1}$ induce different partitions of $\mathcal{V}(G)$}
   \FOR{$v \in \mathcal{V}(G)$}
   \STATE {set $c_{i}(v) = HASH(\mathcal{V}(G_v) \xrightarrow[]{c_{i-1}|_{G_v}} \mathbb{N})$}
   \ENDFOR
   \STATE {i = i+1}
   \ENDWHILE
   \STATE{\textbf{return }$c_i$}
   \end{algorithmic}
\end{algorithm}

\begin{definition}
A coloring of a graph $G$ is a mapping $c:\mathcal{V}(G)\rightarrow \mathbb{N}$, where each $n\in \mathbb{N}$ is thought of as a color. A colored neighborhood of $v\in\mathcal{V}(G)$ is a coloring of $G_v$. Given a coloring $c$ of $G$, restricting it to $G_v$ induces a neighborhood coloring $c|_{G_v}$ for every $v$. We let $\mathcal{G}$ be a set of representatives for the isomorphism classes of neighborhoods (that is, for each $v\in\mathcal{V}(G)$ there is a $K\in \mathcal{G}$ such that $G_v\cong K$). Define the set of neighborhood colourings $CN(G) = \{c:K\rightarrow\mathbb{N}, K\in \mathcal{G}\}$. For a graph $G$ satisfying the finiteness assumption of Remark~\ref{rem:finiteness}, we may take $\mathcal{G}$ to be finite, which ensures the existence of an injective map $HASH:CN(G)\rightarrow \mathbb{N}$. 
\end{definition}

\begin{definition}
The color refinement of a graph $G$ is the coloring process described in Algorithm~\ref{alg:colorrefinement}. It is therefore a sequence of colorings $c_i$ of $G$. See Remark~\ref{rem:multisets} for a comment on the approach using multisets, and Remark~\ref{rem:convergence} for one on convergence of the algorithm. We abuse notation and pretend that it does not depend on $HASH$.
\end{definition}

\begin{remark}\label{rem:multisets}
The usual exposition to the WL test would replace our definition of $CN(G)$ by multisets, where at stage $i$ of Algorithm~\ref{alg:colorrefinement} the multiset associated to $v\in\mathcal{V}(G)$ is $\{c_i(v), \{\{ c_i(u), u\in \mathcal{N}_G(v)\}\}\}$, and the HASH function goes from the set of multisets, e.g. \cite{Morris_Ritzert_Fey_Hamilton_Lenssen_Rattan_Grohe_2019}.
\end{remark}

\begin{remark}\label{rem:convergence}
For finite graphs, Algorithm~\ref{alg:colorrefinement} converges in $\mathcal{O}(|\mathcal{V}(G)|)$. We refer to this converged coloring, if it exists, as the stable coloring of a graph. For infinite graphs, the algorithm does not necessarily converge. However, for an infinite graph that covers a finite graph, the color refinement converges by Lemma~\ref{lem:bigcovering}, making Lemma~\ref{lem:trees} applicable to universal covers defined in Definition~\ref{def:universalcover}.
\end{remark}

\begin{definition}
The WL test between $G$ and $G'$ consists of running the color refinement on the disjoint union $G \bigoplus G'$ until convergence, if both graphs have the same number of nodes of each color in the stable coloring, we say that $G$ and $G'$ are WL equivalent (they succeed WL) otherwise we say that they are not WL equivalent (they fail WL).
\end{definition}

\subsection{Lemmas and Proofs}
In this subsection, we present the main Lemmas and proofs needed for Theorem~\ref{thm:main}. We first recall that the WL test solves the isomorphism problem for trees, whose proof can be found in Appendix~\ref{app:isotrees}.

\begin{lemma}\label{lem:trees} Let $T$ and $T'$ be (possibly infinite) trees with a stable coloring, then $T$ and $T'$ are indistinguishable by WL if and only if they are isomorphic.
\end{lemma}

We show that coverings keep the color refinement invariant:
\begin{lemma}\label{lem:bigcovering}
The color refinement of a graph is fully determined by the color refinement of any graph it covers or any graph covering it.
That is, let $G$, $H$ be graphs, and $\pi: H \rightarrow G$ be a covering map. If $g\in G$ has a given color at step $i$, then any $h\in \pi^{-1}(g)$ has the same color at step $i$. Conversely, if $\pi(h)$ has a given color at step $i$, then $h\in H$ has the same color at step $i$.
\end{lemma}

\begin{proof}
We prove the result inductively.
At the first stage, both graphs are colored by one color so the statement holds.
Suppose the statement holds at stage $i$, then by definition of a cover, the subgraph $H_{h}\subseteq H$ is mapped isomorphically to the subgraph $G_{\pi(h)} \subseteq G$. By assumption the $i$th coloring is fixed by $\pi$, so the neighborhood coloring is the same, and so is the updated colors.
\end{proof}

We have seen that covering maps respect the color refinement of graphs. Hence a graph and its cover look alike from the WL perspective. But for finite graphs, the WL test always distinguishes two graphs with different number of nodes. The dichotomy between order and covers is in fact sufficient to fully characterise WL equivalence classes. To formalize the interplay between covers and order, we introduce the degree of covers, for which we need the following result:

\begin{lemma}\label{lem:degree}
Let $H$ and $G$ be connected graphs, and let $\pi: H\rightarrow G$ be a covering map. For any two $v, \ v'\in \mathcal{V}(G)$, the fibers of $v$ and $v'$ are of the same cardinality, i.e. $|\pi^{-1}(v)|=|\pi^{-1}(v')|$. In particular, $|\mathcal{V}(H)| = |\pi^{-1}(v)|\cdot|\mathcal{V}(G)|$.
\end{lemma}
\begin{proof}
We show that the size of fibers is locally constant which implies the result since the graphs are path connected. Let $v\in \mathcal{V}(G)$, assume that $d=|\pi^{-1}(v)|$, and let $u\in \mathcal{N}_G(v)$. Since $\pi$ is a cover, each $\tilde v \in \pi^{-1}(v)$ has precisely one neighbor $\tilde u \in \mathcal{N}_H(\tilde v)$ with $\pi(\tilde u) = u$, showing that $|\pi^{-1}(u)|\geq|\pi^{-1}(v)|$. By symmetry, swapping $u$ and $v$ gives $|\pi^{-1}(v)|\geq|\pi^{-1}(u)|$ and concludes the proof.
\end{proof}
\newpage
\begin{definition}
Let $H$ and $G$ be connected graphs, and let $\pi: H\rightarrow G$ be a covering map. The degree $d_\pi\in \mathbb{N}$ (or number of sheets) of $\pi$ is the size of any pre-image of a node in $G$ (i.e. $d_\pi=|\pi^{-1}(v))|$), which is well defined by Lemma~\ref{lem:degree}.
\end{definition}

Using the above definition, we can prove a useful statement about the WL test and coverings:

\begin{corollary}\label{cor:combine}
Let $G$ and $G'$ be two connected graphs of equal order. Let $\pi: H\rightarrow G$ and $\pi': H'\rightarrow G'$ be two covers such that $H$ and $H'$ are isomorphic. Then $G$ and $G'$ are WL equivalent.
\end{corollary}
\begin{proof}
Since $H$ and $H'$ are isomorphic, their stable coloring must be the same. By Lemma~\ref{lem:bigcovering}, the coloring of $G$ and $G'$ are therefore the same. Since $G$ and $G'$ have the same order, both covers must have the same degree. We can conclude that the number of nodes of a given color in $G$ (resp $G'$) is given by the number of nodes of that color in $H$ (resp $H'$) divided by $d$.
\end{proof}
At last, we introduce the universal cover of a graph $G$:

\begin{definition}\label{def:universalcover}
The universal cover of a connected graph $G$ is a cover $\pi: T \rightarrow G$ where $T$ is a connected tree. This map is unique up to cover isomorphism, and covers all other covers of $G$ \cite{Hatcher:478079}. See Figure~\ref{fig:universalcover} for an illustration.
\end{definition}

\begin{figure}[ht]
    \vskip 0.2in
    \begin{center}
    \centerline{\includegraphics[width=0.8\columnwidth]{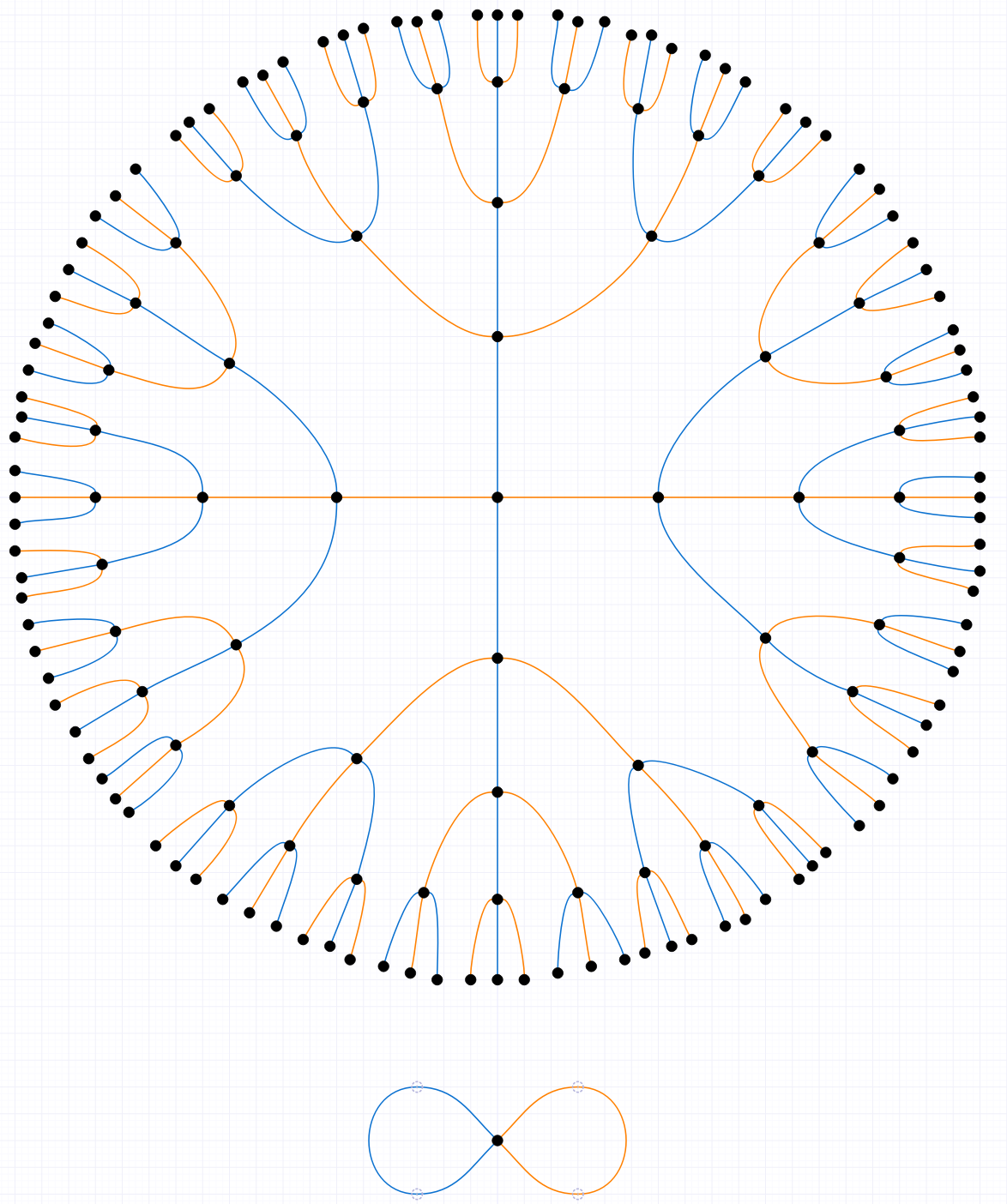}}
    \caption{Finite subgraph of the infinite universal cover (top) of the rose with $2$ petals (bottom). We apologies since the rose is not a graph as per Definition~\ref{def:graphs}. To make it into a graph we would have to subdivide each edge twice, and do the same each edge of the universal cover. The finite subgraph of the universal cover is the $4$ hop neighborhood of the central node.}
    \label{fig:universalcover}
    \end{center}
    \vskip -0.2in
\end{figure}
\begin{remark}
If the Euler characteristic $\chi(G) :=|\mathcal{V}|-|\mathcal{E}| $, of a finite connected graph $G$ is $\leq 0$ (i.e. $G$ has a loop), then the universal cover is infinite, otherwise $G$ is a tree.
\end{remark}
Finally, we prove the main Lemma for the second direction of Theorem~\ref{thm:main}, and then prove the main theorem.

\begin{lemma}\label{lem:universalcover}
Two connected graphs that are not distinguishable by WL have isomorphic universal covers.
\end{lemma}

\begin{proof}
Let $G$ and $G'$ be two connected graphs which are not distinguishable by WL. Let $\pi: T \rightarrow G $ and $\pi': T' \rightarrow G'$ be their respective universal covers. By Lemma~\ref{lem:bigcovering}, since $G$ and $G'$ are indistinguishable by WL, so are $T$ and $T'$. By Lemma~\ref{lem:trees}, $T$ and $T'$ must therefore be isomorphic.
\end{proof}

\begin{proof}[Proof of Theorem~\ref{thm:main}]
One direction is obtained from Lemma~\ref{lem:universalcover}, the other from Corollary~\ref{cor:combine}.
\end{proof}

\section{The GraphCovers dataset}
Our goal is to efficiently generate graphs that are indisinguishable by the WL test, but are pairwise non-isomorphic. To do so, Theorem~\ref{thm:main} indicates that same degree covers of a common base graph are good candidates since they all share the same universal cover, and hence fool the WL test. However, we still need a way to distinguish non-isomorphic covers, and two covers can be non-isomorphic as covers but still be isomorphic as graphs. After showing that under mild assumptions on the base graph, both notions of isomorphism agree, we provide an efficient algorithm to test isomorphism of covers. We combine these into Algorithm~\ref{alg:gendataset} which returns GraphCovers$(G, d)$, a set of representatives for the isomorphism classes of connected, degree $d$ covers of a base graph $G$.

\begin{theorem}\label{thm:covisom}
If color refinement distinguishes every vertex of a base graph $G$ (i.e. the stable coloring $c:\mathcal{V}(G)\rightarrow \mathbb{N}$ is injective) then two covers $\pi:H\rightarrow G$ and $\pi':H' \rightarrow G$ are isomorphic as covers if and only if they are isomorphic as graphs.
\end{theorem}

\begin{proof}
If two covers are isomorphic, then the underlying graphs are isomorphic. We thus focus on the other direction.
Let $\phi: H\rightarrow H'$ be an isomorphism of graphs. By Lemma~\ref{lem:bigcovering}, all nodes in $\pi^{-1}(v)$, and all nodes in $\pi'^{-1}(v)$ have the same color as $v$. Since isomorphisms are covers, colorings have to be preserved by the $\phi$ per Lemma~\ref{lem:bigcovering}. Since every node in $G$ have different colors, we must have that $\phi: H\rightarrow H'$ maps $\pi^{-1}(v)$ to $\pi'^{-1}(v)$, thus making the diagram in Definition~\ref{def:coveringmorphism} commute. 
\end{proof}

In order to build the promised algorithm, we need a way to check if two covers are isomorphic as covers. Lemma~\ref{lem:coverextendability} shows that this problem can be reduced to checking if we can extend a partial map defined on a single point to an isomorphism. In Proposition~\ref{prop:algocovisom} we show that extending such maps can be done in polynomial time.

\begin{lemma}\label{lem:coverextendability}
A morphism $\phi: H\rightarrow H'$ between connected coverings $\pi: H\rightarrow G$ and $\pi':H'\rightarrow G'$ of a connected base graph $G$ is uniquely determined by its value at any vertex.
\end{lemma}
\begin{proof}
As for Lemma~\ref{lem:degree} we work locally, by showing that there is no choice when extending a partial $\tilde v \mapsto \tilde v'$ to the neighbors $\tilde u\in \mathcal{N}_H(\tilde v)$. Let $\tilde v\in\mathcal{V}(H)$, and $\tilde v':=\phi(\tilde v)$, and let $\tilde u \in \mathcal{N}_H(\tilde v)$. Since $\pi'$ is a covering map, there is a unique $\tilde u' \in \mathcal{N}_{H'}(v') \cap \pi'^{-1}(\pi(\tilde u))$, and thus we must send $\tilde u$ to this $\tilde u'$ to preserve commutativity of the diagram.
\end{proof}

\begin{proposition}\label{prop:algocovisom}
Let $G$ be a connected graph satisfying the conditions of Theorem~\ref{thm:covisom}. There is a $\mathcal{O}(d^2\cdot\mathcal{E}(G))$ algorithm that checks if two degree $d$ covers of the base graph $G$ are isomorphic as graphs.
\end{proposition}
\begin{proof}
Let $ \pi: H\rightarrow G$ and $\pi': H' \rightarrow G$ be two covers of degree $d$. By Theorem~\ref{thm:covisom}, it suffices to check if there exists a cover isomorphism between $\pi$ and $\pi'$. Fix a $\tilde v\in \mathcal{V}(H)$. By Lemma~\ref{lem:coverextendability}, it suffices to check if for some $\tilde v'\in {\pi'}^{-1}(\pi(v))$, the map $\tilde v\mapsto \tilde v'$ can be extended to a covering isomorphism from $H$ to $H'$. This can be done by extending the map locally edge by edge, and checking that at each stage it is an isomorphism onto its image. Looping through every such $\tilde{v'}$ gives a $\mathcal{O}(d\cdot\mathcal{E}(H))$ algorithm.
\end{proof}

\begin{algorithm}[tb]
\caption{Generate GraphCovers}
\label{alg:gendataset}
\begin{algorithmic}
   \STATE {\bfseries Inputs:} Graph $G$ as in Theorem~\ref{thm:covisom}, Degree $d>0$\\
   \STATE Initialize GraphCovers$(G, d)=\emptyset$.
   \FOR{every connected degree $d$ cover $\pi: H\rightarrow G$}
   \IF{$H\cong H'$ for some $H'\in$ GraphCovers$(G, d)$}
   \STATE{do nothing}
   \ELSE
   \STATE{GraphCovers$(G, d).$add($H$)}
   \ENDIF
   \ENDFOR
   \STATE {\bfseries return } GraphCovers$(G, d)$
   \end{algorithmic}
\end{algorithm}

To implement Algorithm~\ref{alg:gendataset}, we enumerate the degree $d$ covers of $G$ by selecting one permutation $\sigma_e \in S_d$ for each edge $e\in \mathcal{E}$, where $S_d$ is the group of permutations on $d$ elements. The cover $H$ is then constructed by creating $d$ copies $v_1, v_2, \ldots v_d$ for each $v \in \mathcal{V}$, and for each $e=vw\in \mathcal{E}$ we add edges connecting $v_{i}$ to $w_{\sigma_e(i)}$ for $i=1\ldots d$. This procedure is used in \cite{GROSS1977273} (see Theorem~$2$ in that paper), and in \cite{randomcovers2002}. To capture all isomorphism classes it is enough to fix $\sigma_e=id$ except for certain distinguished edges, where the number of distinguished edges is equal to $1-\chi(G)$. We use the algorithm described in Proposition~\ref{prop:algocovisom} to check if two generated covers are isomorphic.

The limits of Algorithm~\ref{alg:gendataset} is that it generates only connected covers, and cannot generate a pair of graph with common cover that do not cover a common base graph (such as ones in Figure~\ref{fig:coveringexample}). Whereas the first point can be addressed, hence generalizing the construction of the CYCLES dataset from \cite{Horn22a}; the second cannot.

\section{A poly-factorial lower bound to the size of WL equivalence classes}
In the previous section, we gave an efficient algorithm that generates GraphCovers$(G, d)$, a collection of graphs that fool the WL test but that are pairwise non-isomorphic. However, we do not know how big this collection is.  In the case where $G$ is a length $n$ cycle, any degree $d$ covering of $G$ isomorphic to the cycle of degree $d\cdot n$, so our dataset would consist of only one graph, making it useless. This suggests that for GraphCovers$(G, d)$ to be useful, the topology of the graph $G$ has to be more complicated than that of a cycle. We thus show that if $G$ has two or more cycles, then GraphCovers$(G, d)$ grows factorially with $d$.

In Corollary~\ref{cor:nbcovers} we reduce the problem to the one of counting index $d$ subgroups in a rank $r$ free group, where $r = 1-\chi(G)$, and $\chi(G) = |\mathcal{V}| - |\mathcal{E}|$ is the Euler characteristic of the graph. This problem was solved in \cite{hall_1949} leading to the following recursive formula:

\begin{theorem}\cite{hall_1949}\label{thm:hall}
The number $N_{d, r}$ of subgroups of index $d$ in the rank $r$ free group $F_r$ is given recursively by $N_{1, r}=1$, 
\[N_{d, r} = d(d!)^{r-1} - \sum_{i=1}^{d-1}[(d-i)!]^{r-1}N_{i,r}\ .\]
\end{theorem}

The following is a well know fact from algebraic topology that we need, but do not prove in this paper. See Chapter $1$ in \cite{Hatcher:478079} for the definition of the fundamental group, and Proposition~$1$A.$2$ for a proof of the following.

\begin{proposition}\label{prop:fundgroup}
The fundamental group $\pi_1(G, p)$ based at $p\in \mathcal{V}(G)$ of a connected graph $G$ is isomorphic to the free group $F_r$ of rank $r=1-\chi(G)$.
\end{proposition}

The exact number of connected degree $d$ covers was computed in \cite{ctx1949787780005516}, but for our purpose we simply need a good lower bound. We combine Theorem~\ref{thm:hall} and Proposition~\ref{prop:fundgroup} to obtain the following lower bound on the number of non-isomorphic degree $d$ connected covers of a connected base graph $G$, and therefore a lower bound on the size of GraphCovers($G$, $d$). The proof can be found in Appendix~\ref{app:bound}.
\newpage
\begin{corollary}\label{cor:nbcovers}
The number $C_{d, r}$ of non-isomorphic connected degree $d$ covers of a connected graph $G$ of Euler characteristic $0>\chi(G) = 1 - r$ satisfies \[
    d\cdot C_{d, r} \geq N_{d, r}
\]
where $N_{d, r}$ is the number of index $d$ subgroups of the free group of rank $r$. In particular, 
\[C_{d, r}\geq d^{r-2}((d-1)!)^{r-1}\  .\]
\end{corollary}
\section{Experiments}

We now design experiments to benchmark our dataset. First, we take three connected non-isomorphic degree $5$ covers of a connected base graph $G$ with $|\mathcal{V}(G)|=9$ and $\chi(G)=-1$ that satisfy the conditions of Theorem~\ref{thm:covisom}, see Figure~\ref{fig:basegraph} in Appendix~\ref{app:graphs} for an illustration of the chosen $G$. All three resulting connected graphs have $45$ nodes and $50$ edges. The choices of the $3$ graphs and the degree $5$ was arbitrary, and repeating the experiment with different choices always gave the same results. We can now benchmark the dataset by measuring the performance of popular GNN architectures on the supervised graph classification task.
\begin{table}[t]
\caption{Training accuracy of $2$-layer deep nonlinear GNNs on $3$ non-isomorphic, degree $5$ covers of a graph with $\chi(G)=-1$.}
\label{tab:exp2layer}
\vskip -0.3in
\begin{center}
\begin{small}
\begin{sc}
\begin{tabular}{lcccr}
\toprule
MODEL & Unique id & Random & Degree & Constant \\
\midrule
WL & N/A & N/A & $33\%$ & $33\%$ \\
GCN & $100\%$ & $100\%$ & $33\%$ & $33\%$\\
CHEB & $100\%$ & $100\%$ & $33\%$ & $33\%$\\
GIN & $100\%$ & $100\%$ & $33\%$ & $33\%$\\
SAGE & $100\%$ & $100\%$ & $33\%$ & $33\%$\\
GAT & $100\%$ & $100\%$ & $33\%$ & $33\%$\\
\bottomrule
\end{tabular}
\end{sc}
\end{small}
\end{center}
\vskip -0.2in
\end{table}

We consider five GNN architectures, each differing only by their message passing scheme. They consist of two message passing layers with ReLU nonlinearity, followed by global poolings (stack of mean and max pooling) and an MLP with output dimension $3$. The five message passing schemes are ChebConv \cite{deferrard}, GCN \cite{GCN}, GraphSAGE \cite{SAGE}, GAT \cite{GAT}, and GIN \cite{GIN}.
 
We consider four different node features $\mathbf{x}:\mathcal{V}\rightarrow \mathbb{R}^c$. The constant signal equal to $1$ everywhere, the degree signal mapping $v$ to its degree, a random signal sampled from a $1$-dimensional Gaussian $\mathcal{N}(0, 1)$ distribution, and a unique identifier signal where each node is mapped to a $45$-dimensional binary vector of all $0$s except $1$ at a position uniquely determined by the node.

To make sure that the GNNs have the capacity to distinguish these graphs, the hyperparameters are chosen such that all architectures reach $100\%$ training accuracy on the random and unique ID signals. The existence of such a choice of hyperparameters is expected \cite{Loukas2020What, NEURIPS2020_23685a24, NEURIPS2020_75877cb7, 10.5555/3524938.3525258}. The chosen hyperparameters are: cross entropy loss; Adam optimizer with learning rate $0.01$ and weight decay $5\cdot10^{-4}$; hidden dimension $100$; and each architecture is trained for $400$ epochs on the $3$ graphs. We report training accuracy in Table~\ref{tab:exp2layer}, to which we have also added the accuracy of the WL test on these graphs.

The architectures all perform poorly, as they learn only one of the three isomorphism classes. In fact, all the GNNs give constant predictions, since every graph's embedding in the final layer is the same, as expected by Theorem~\ref{thm:main}. 

\section{Conclusion}
After giving a novel description of WL equivalence classes to the GNN community, we materialized the result to provide an arbitrarily large dataset on which message passing GNNs cannot do graph classification.

Our work strengthens the existing link between Algebraic Topology, Group Theory, and Geometric Deep Learning. We showed that GNNs capture two orthogonal structures of graphs: size (number of nodes) and shape (isomorphism type of covers). We also showed that messages propagate equivalently on a graph and its (hyperbolic) universal cover, which might lead to a novel viewpoint on bottlenecks.

Future work includes benchmarking more complex architectures on GraphCovers, and studying the performance of higher order WL tests on GraphCovers.

GraphCovers can now be used to benchmark and study novel GNN architectures. We consider the contribution of a concrete dataset for which message passing does not work to be an essential step in improving GNNs expressiveness, and to go beyond message passing.
We thus conclude by challenging the community to solve graph classification on the GraphCovers dataset for large degree covers.

\section*{Software and Data}
Code to generate GraphCovers$(G, d)$ and to reproduce our experiments can be found on our github repository \footnote{https://github.com/jacobbamberger/GraphCovers}. Experiments were mainly done using pytorch geometric \cite{pytorchgeometric} and Networkx \cite{networkx}.

\section*{Acknowledgements}
I would like to thank the members of LTS$2$ for the stimulating environment, and particularly to Pierre Vandergheynst for the wonderful supervision and generous support. I am grateful to the reviewers for several helpful comments.

\newpage

\bibliography{example_paper}
\bibliographystyle{icml2022}

\newpage
\appendix
\onecolumn

\section{WL test solves isomorphisms for Trees}\label{app:isotrees}

\begin{replemma}{lem:trees}
Let $T$ and $T'$ be (possibly infinite) trees with a stable coloring, then $T$ and $T'$ are indistinguishable by WL if and only if they are isomorphic.\end{replemma}

\begin{proof}[Proof of Lemma~\ref{lem:trees}]
Consider two trees $T$ and $T'$ with same color refinement and with bounded degree as defined in Remark~\ref{rem:finiteness}.
We build the isomorphism inductively, starting at a node $v\in \mathcal{V}(T)$ and defining it on a connected subtree by adding a neighbor to the previous subtree. The order in which we visit the neighbors is that of a breadth first search (BFS) algorithm, ensuring that all neighbors of a given vertex are added before moving to neighbors of a new vertex. By the local finiteness assumption, this ensures local surjectivity. At each step we ensure that it is an isomorphism onto its image, and that it keeps the coloring invariant. In the base case we can send $v$ to a $v'\in\mathcal{V}(T')$ of the same color. In the inductive step, assume we built a map on a connected proper subtree $T_k$ of size $k$, that maps isomorphically onto its image $T_k'$ and keeps color invariant. Consider a vertex $v_{k+1}$ with a neighbor $u\in \mathcal{V}(T_k)$. Since the color refinement of the two trees are the same and $\phi_k$ is an isomorphism between $T_k$ and $T_k'$ that respects coloring, $\mathcal{N}_{T_k}(u)$ and $\mathcal{N}_{T_k'}(\phi_k(u))$ have the same number of nodes of every color, so at least one neighbor $v_{k+1}' \in \mathcal{N}_{T_k'}(\phi_k(u))$ has to have the same color as $v_{k+1}$ and not be in $T_k'$.
We can thus extend $\phi_k$ by defining $\phi_{k+1}(v_{k+1}):=v_{k+1}'$. Since all we did was to grow a branch on $T_k$ starting at $v_k$ and growing a branch on $T_{k+1}'$ starting at $\phi(v_k)$ of same color, the resulting map is an isomorphism between $T_{k+1}$ and $T_{k+1}'$ that keeps coloring invariant. Moreover, local finiteness and the BFS order in which we chose neighbors ensure that all neighbors in $T$ of $v\in \mathcal{V}(T_k)$ will be added to the map in a finite number of steps, and that any $v'\in \mathcal{V}(T')$ is in the image of $\phi_k$ for some finite $k$. Therefore $\bigcup_{k\in \mathbb{N}} T_k = T$ and $\bigcup_{k\in \mathbb{N}} T'_k = T'$ and the resulting map is an isomorphism.
\end{proof}

\section{The disconnected case}\label{app:maincor}

We first state and prove a Lemma that we need for the disconnected case:

\begin{lemma}\label{lem:disjoint}
Stable colorings of two finite connected graphs are either equal or disjoint.
\end{lemma}

\begin{proof}
Let $G$ and $G'$ be finite connected graphs with differing stable colorings. Since they are finite graphs, there must be one color which is only in one of the two colorings. Let $v\in\mathcal{V}(G)$ be a node of that color. Since $v$'s color does not appear in $G'$, then by injectivity of the updating function $HASH$, $v$'s neighbors' colors also do not appear in $G'$. Path connectedness of $G$ allows us to conclude.
\end{proof}

We can now prove the main theorem for disconnected graphs:
\begin{reptheorem}{cor:disconnected}
Two graphs $G$ and $G'$ of bounded degree are indistinguishable by the WL test if and only if they admit decompositions into disjoint unions of connected components:
    \[ G = \bigoplus_{i\in \mathcal{I}} \bigoplus_{j\in\mathcal{J}_i} G_{ij} \text{ and } G' = \bigoplus_{i\in\mathcal{I}}\bigoplus_{j'\in\mathcal{J}_{i}'} G'_{ij'} \ ,\] such that the following holds
    \begin{itemize}
    \vspace{-0.15in}
        \item for every $i\in\mathcal{I}$,  $|\bigoplus_{j\in\mathcal{J}_i} G_{ij}| =  |\bigoplus_{j'\in\mathcal{J}'_{i}} G'_{ij'}|$
        \vspace{-0.075in}
        \item for every $i\in\mathcal{I}$, there is a graph $H_i$ that covers both $G_{ij}$ and $G'_{ij'}$ for all $j\in\mathcal{J}_i$ and $j'\in\mathcal{J}'_i$.
    \end{itemize}
    
    Moreover, if $|\bigoplus_{j\in\mathcal{J}_i} G_{ij}|<\infty$, the cover $H_i$ can be chosen to be finite.
\end{reptheorem}

\begin{proof}[Proof of Theorem~\ref{cor:disconnected}]
Let $G$ and $G'$ be graphs that are indistinguishable by the WL test. Let $\Gamma$ be a connected component of $G$. By Lemma~\ref{lem:disjoint}, there must be a connected component $\Gamma'$ in $G'$ with the same coloring as $\Gamma$. Repeating the argument for each connected component of $G$ gives a bijection between connected components of $G$ and connected components of $G'$. Applying Lemma~\ref{lem:disjoint} to the components of $G$ allows us to organize the components in groups of colorings, where all components having the same coloring are indexed by some $i\in \mathcal{I}$. We thus obtain decompositions: 
\[ G = \bigoplus_{i\in\mathcal{I}}\bigoplus_{j\in\mathcal{J}_i}G_{ij} \text{ and } G' = \bigoplus_{i\in\mathcal{I}}\bigoplus_{j'\in\mathcal{J}'_i}G_{ij'}. \]
Moreover, for fixed $i, \ i' \in\mathcal{I}$, every $G_{ij}$ and every $G_{i'j'}$ have the same coloring if and only if $i=i'$, so by Theorem~\ref{thm:main} they admit a common cover if and only if $i=i'$. To see that for every $i\in\mathcal{I}$,  $|\bigoplus_{j\in\mathcal{J}_i} G_{ij}| =  |\bigoplus_{j'\in\mathcal{J'}_{i}} G_{ij'}|$, it suffices to look at the first time step of the color refinement on both collections, since by Lemma~\ref{lem:disjoint} and our choice of indexing by colorings, different $i$s have disjoint colorings. Therefore $\bigoplus_{j\in\mathcal{J}_i} G_{ij}$ and $\bigoplus_{j'\in\mathcal{J'}_{i}} G_{ij'}$ are indistinguishable by WL, and they must have the same coloring at every step of the color refinement, in particular they must have the same coloring at step $0$, but this means that they have the same .

For the converse, we assume that $G$ and $G'$ admit such decompositions. It suffices to show that for every $i\in\mathcal{I}$, the graphs  $\bigoplus_{j\in\mathcal{J}_{i}} G_{ij}$ and $\bigoplus_{j'\in\mathcal{J'}_{i}} G_{ij'}$ are indistinguishable by WL. For two $j_1, \ j_2 \in \mathcal{J}_i$, since $G_{ij_1}$ and $G_{ij_2}$ both have a common cover with $G'_{ij'}$ for any $j'\in\mathcal{J}_i'$, they must have a common cover, and theorefore the same coloring by Theorem~\ref{thm:main}. Hence there is a graph $H_i$ that covers all $G_{ij}$s for $j\in\mathcal{J}$ and also all $G_{ij'}$s for $j'\in\mathcal{J}'$, we denote the degrees of these covers by $d_j$ and $d'_{j'}$ respectively. In particular, the number of nodes of a specific color at any step of the color refinement of $H_i$ is precisely $\sum_{j\in\mathcal{J}_i}d_j$ times the number of nodes of that color in $\bigoplus_{j\in\mathcal{J}_{i}} G_{ij}$, but also   $\sum_{j'\in\mathcal{J'}_i}d'_{j'}$ times the number of nodes of that color in $\bigoplus_{j'\in\mathcal{J'}_{i}} G_{ij'}$. However, since the  of both graphs are equal, we get that the colorings at step $0$ are the same, and therefore $\sum_{j\in\mathcal{J}_i}d_j = \sum_{j'\in\mathcal{J'}_i}d'_{j'}$. We thus conclude that both graphs are indistinguishable by WL using Theorem~\ref{thm:main}.
\end{proof}


\section{Lower bound on the size of WL equivalence classes}\label{app:bound}

\begin{repcorollary}{cor:nbcovers}
The number $C_{d, r}$ of non-isomorphic degree $d$ covers of a graph $G$ of Euler characteristic $0>\chi(G) = 1 - r$ satisfies \[
    d\cdot C_{d, r} \geq N_{d, r}
\]
where $N_{d, r}$ is the number of index $d$ subgroups of the free group of rank $r$. In particular: 
\[C_{d, r}\geq d^{r-2}((d-1)!)^{r-1} .\]
\end{repcorollary}

\begin{proof}[Proof of Corollary~\ref{cor:nbcovers}]
Let $r=1-\chi(G)$, and pick a base-point $p\in \mathcal{V}(G)$ allowing us to identify $\pi_1(G, p)$ to the rank $r$ free group $F_r$. Index $d$ subgroups of $F_r$ are in bijection with isomorphism classes of degree $d$ based covering maps of $G$ (Proposition $1.32$ and Theorem~$1.38.$ in \cite{Hatcher:478079}), where a based covering of $G$ is a covering of $G$ together with a choice of basepoint $\tilde p \in \pi^{-1}(p)$. Moreover, every cover corresponds to at most $d$ non-isomorphic based covering maps differing only by the choice of base-point $\tilde p \in \pi^{-1}(p)$. This shows that $d\cdot C_{d,r}\geq N_{d, r}$.\\
To see that $C_{d, r}$ grows factorially we show that $N_{d, r}$ does when $r\geq 2$.  We show that $N_{d, r}\geq (d!)^{r-1}$:
\begin{align*}
    N_{d, r} & = (d!)^{r-1} + \sum_{i=1}^{d-1}(d!)^{r-1} - [(d-i)!]^{r-1}N_{i,r}\\
    & \geq (d!)^{r-1} + \sum_{i=1}^{d-1}(d!)^{r-1} - [(d-i)!]^{r-1}i(i!)^{r-1}\\
    & \geq (d!)^{r-1} \ .
\end{align*}
Where the first inequality follows from $N_{i, r}\leq i(i!)^{r-1}$, and in the last line we use that for any $i$, 
\[(d!)^{r-1} - i[(d-i)!]^{r-1}(i!)^{r-1}\geq 0.\] To show this, we may assume WLOG that $i\leq\frac{d}{2}$ (otherwise repeat the argument with $d-i$ instead of $i$). From which we have 
\begin{align*}
    \frac{d!}{(d-i)!} &= (\frac{d}{2}\cdot2)(\frac{d-1}{2}\cdot2)\ldots (\frac{d-i+1}{2}\cdot2) \\
    &\geq 2^{i} \cdot i! \geq i \cdot (i!) \ , 
\end{align*}
where the first inequality follows from the assumption $\frac{d-k}{2}\geq \frac{d}{2}-k\geq i-k$.
The result follows from raising both sides to the power $r-1$.
\end{proof}


\newpage
\section{Examples of covers}\label{app:graphs}
In this Section we illustrate examples of covering spaces. We fix a base graph $G$ in Figure~\ref{fig:basegraph}. It is obtained from the bottom left graph in Figure~\ref{fig:coveringexample} by adding $3$ nodes to ensure that the conditions of  Theorem~\ref{thm:covisom} are satisfied. In Figures~\ref{fig:deg2covers} and \ref{fig:deg3covers} we illustrate GraphCovers$(G, 2)$ and GraphCovers$(G, 3)$ respectively. All colours are stable and obtained using a fixed $HASH$.

\begin{figure*}[ht]
    \begin{subfigure}
        \centering
        \includegraphics[width=0.22\textwidth]{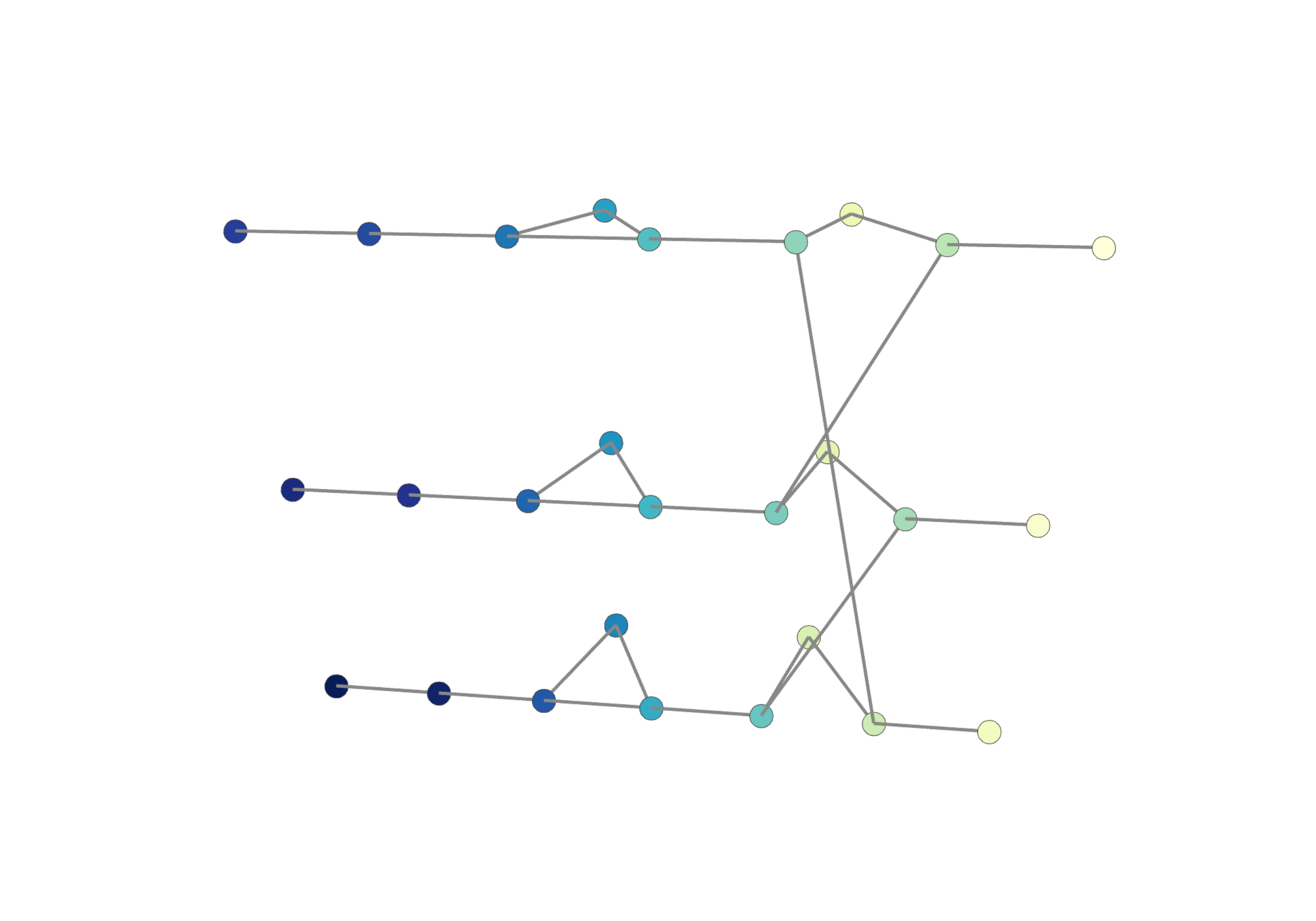}
        \end{subfigure}
        \hfill
    \begin{subfigure}
        \centering
        \includegraphics[width=0.22\textwidth]{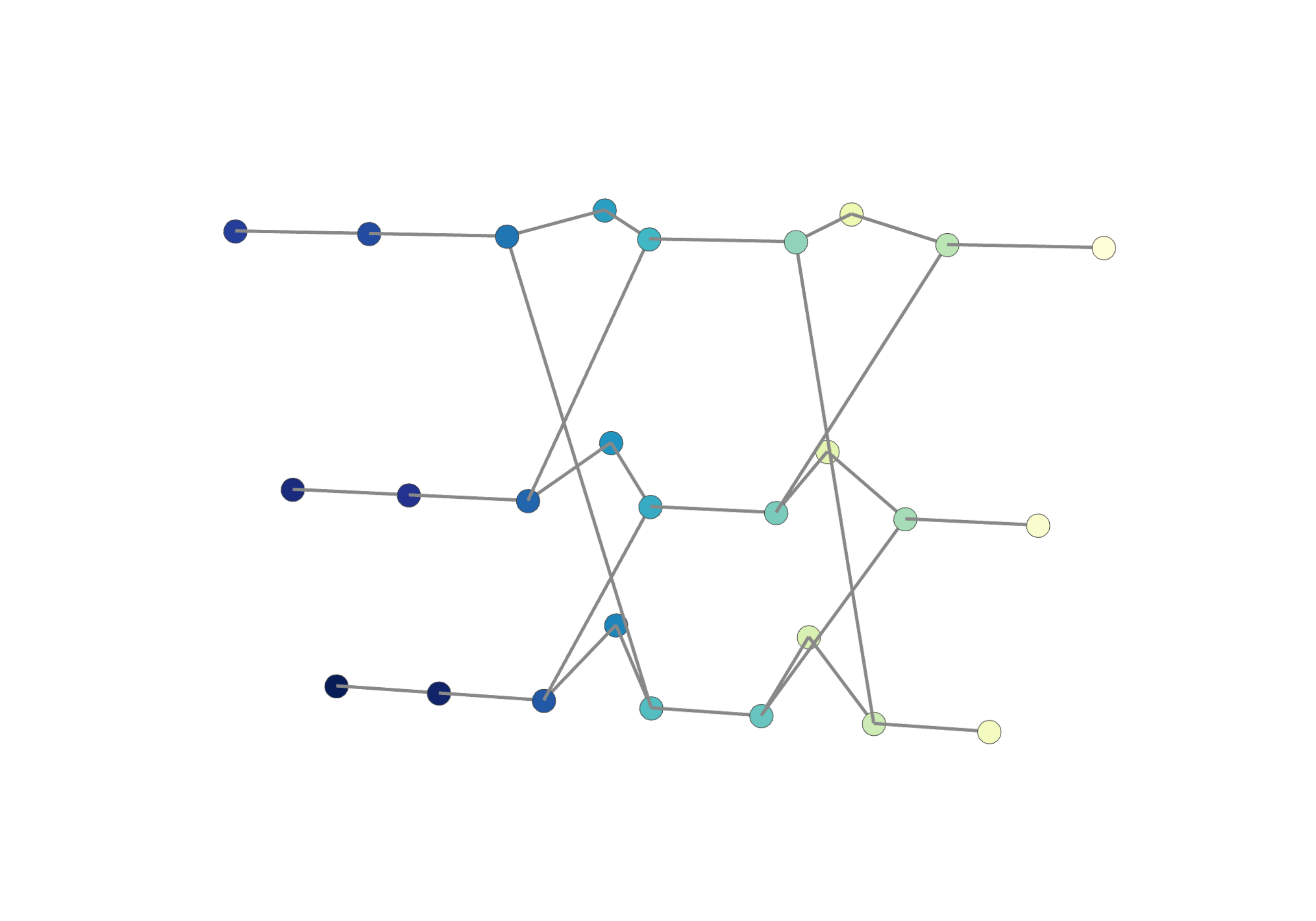}
    \end{subfigure}
    \hfill
    \begin{subfigure} 
        \centering
        \includegraphics[width=0.22\textwidth]{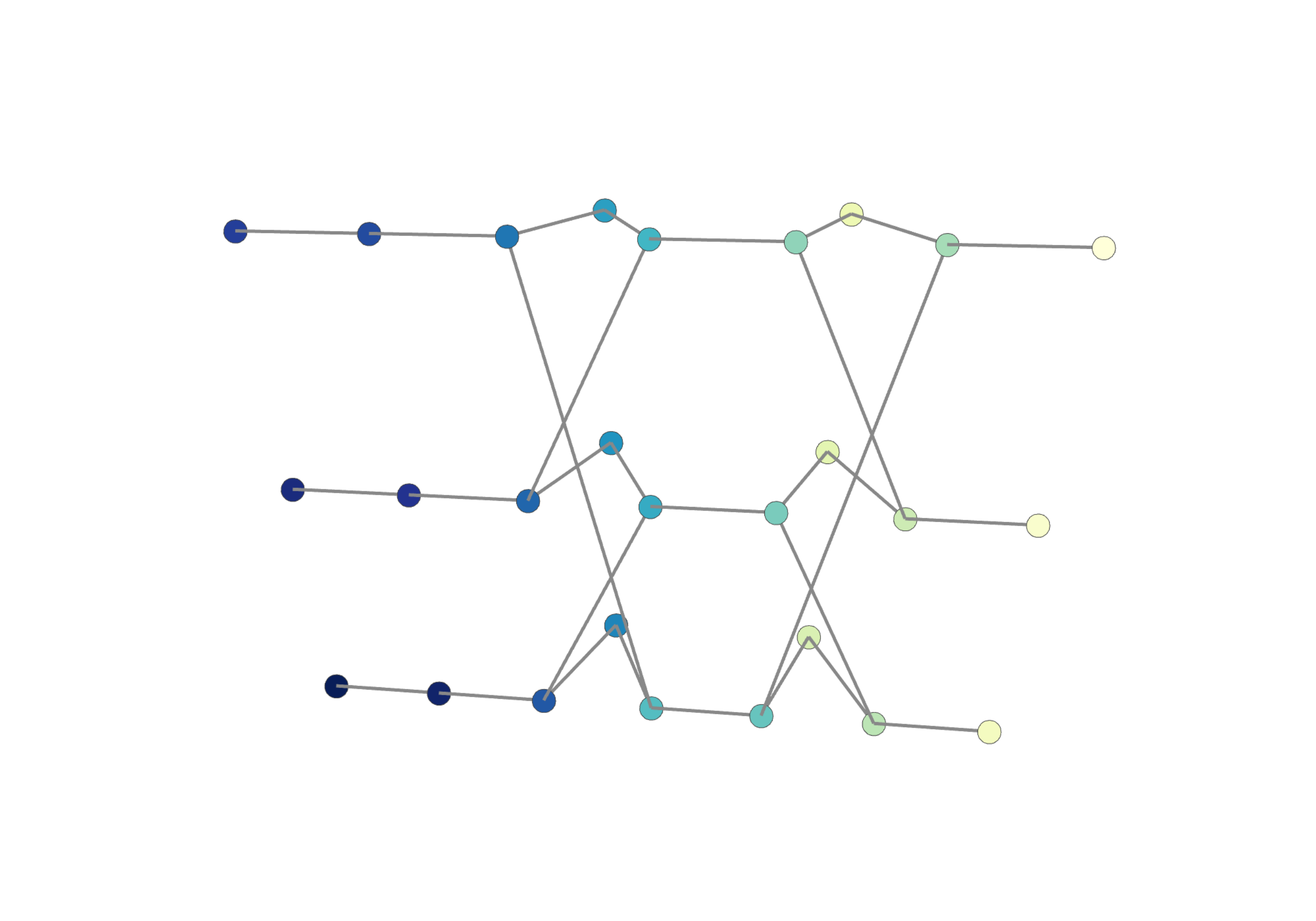}
    \end{subfigure}
        \hfill
    \begin{subfigure} 
        \centering
        \includegraphics[width=0.22\textwidth]{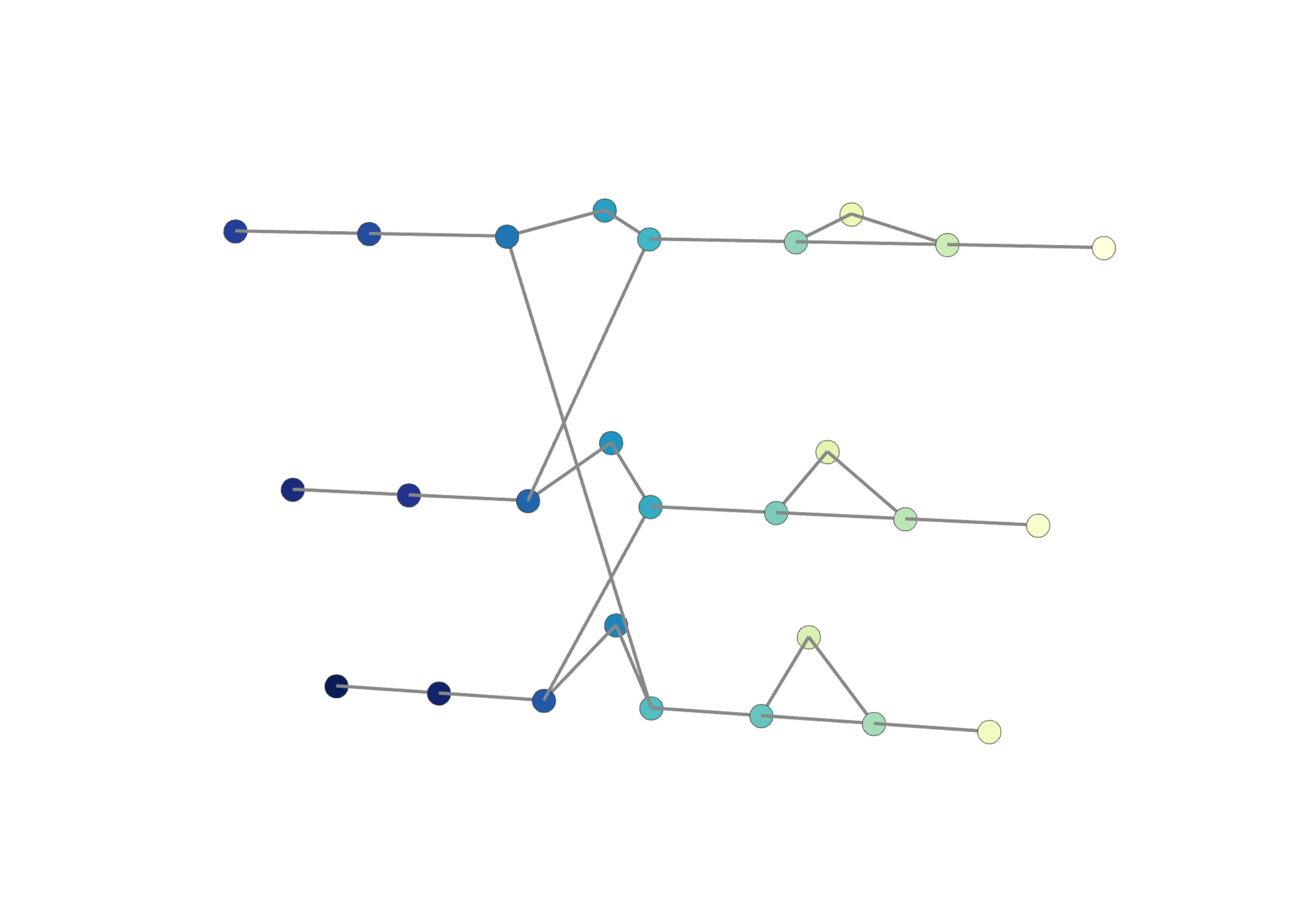}
    \end{subfigure}
        \hfill
    \begin{subfigure} 
        \centering
        \includegraphics[width=0.22\textwidth]{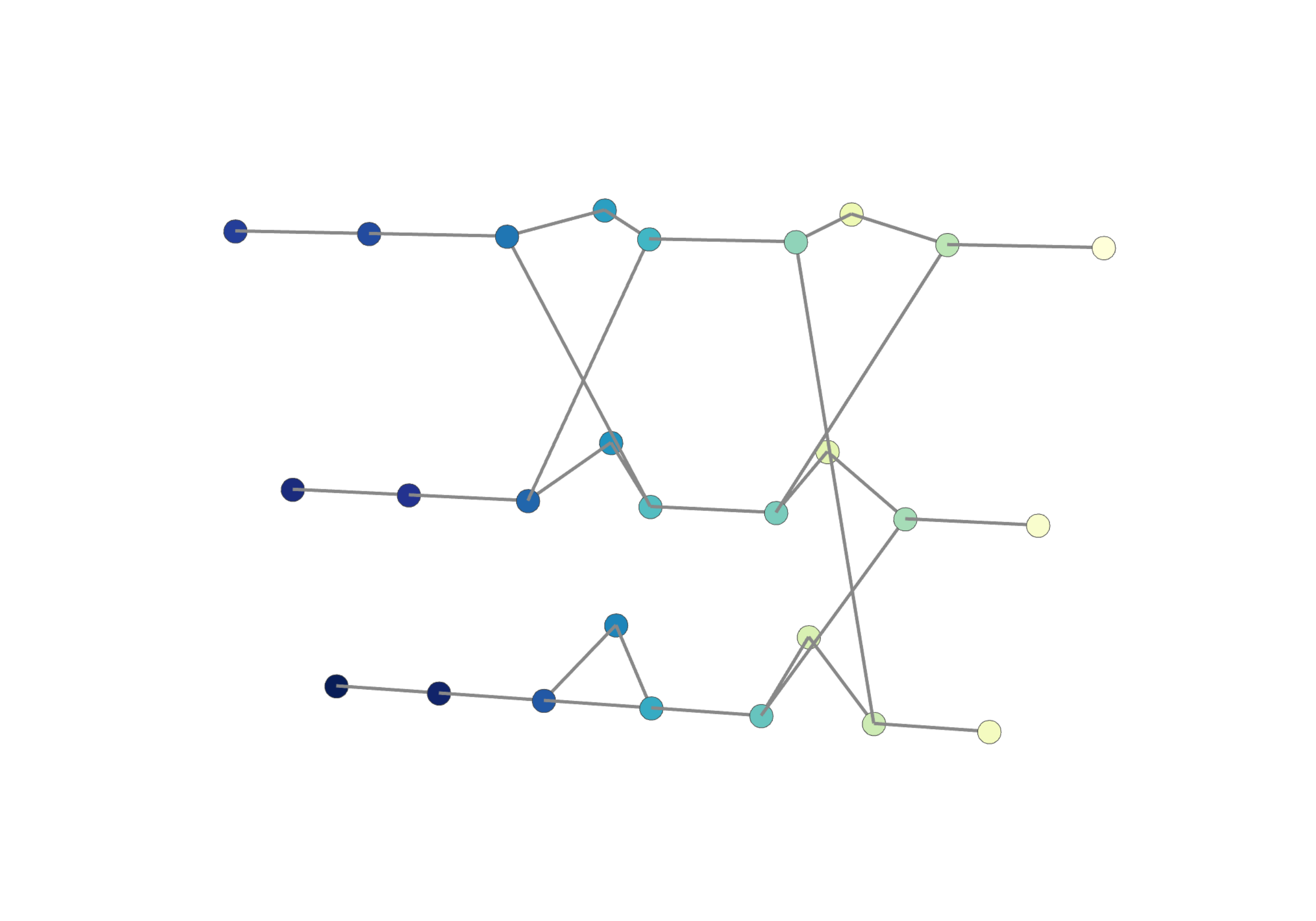}
    \end{subfigure}
        \hfill
    \begin{subfigure} 
        \centering
        \includegraphics[width=0.22\textwidth]{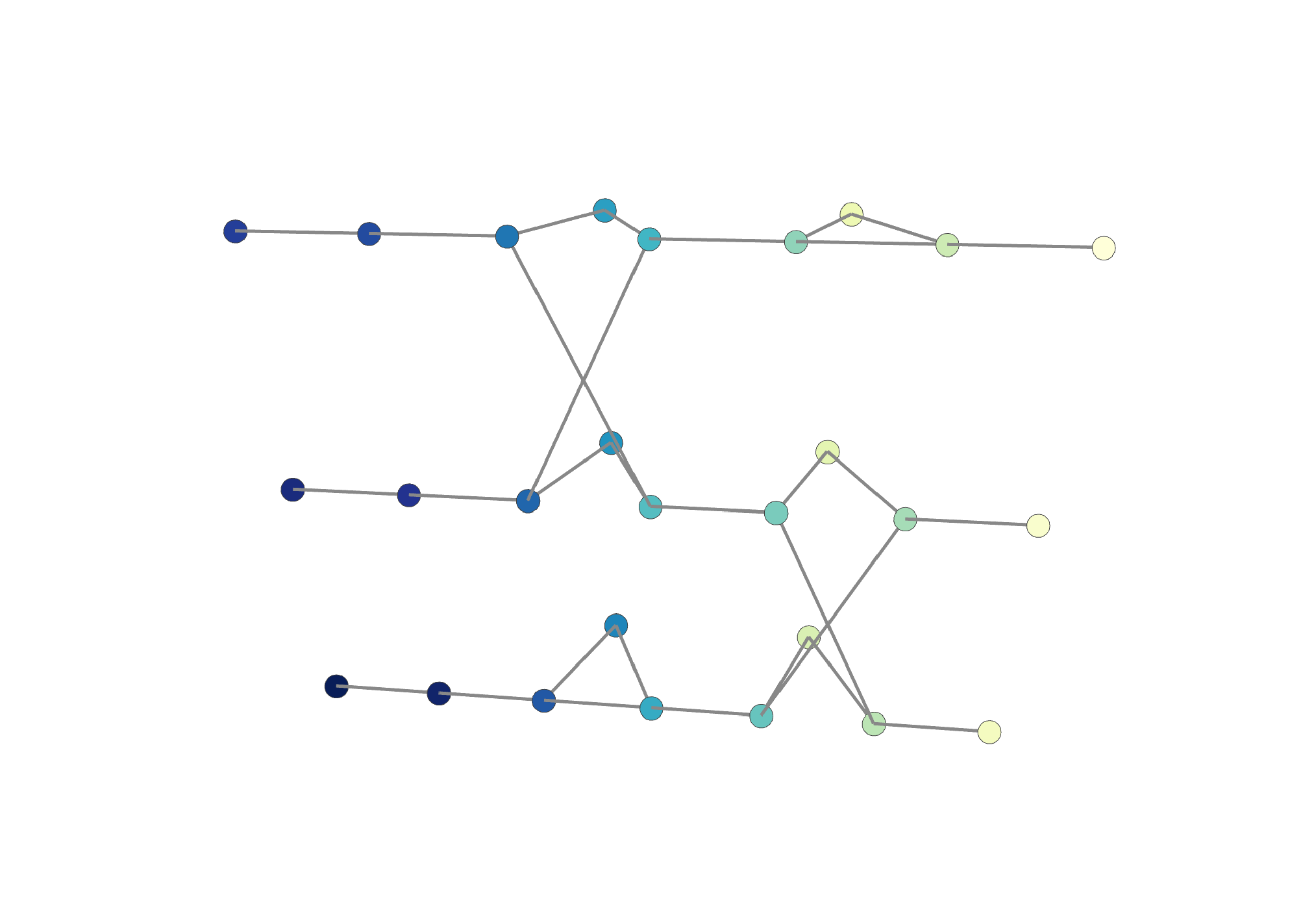}
    \end{subfigure}
        \hfill
    \begin{subfigure} 
        \centering
        \includegraphics[width=0.22\textwidth]{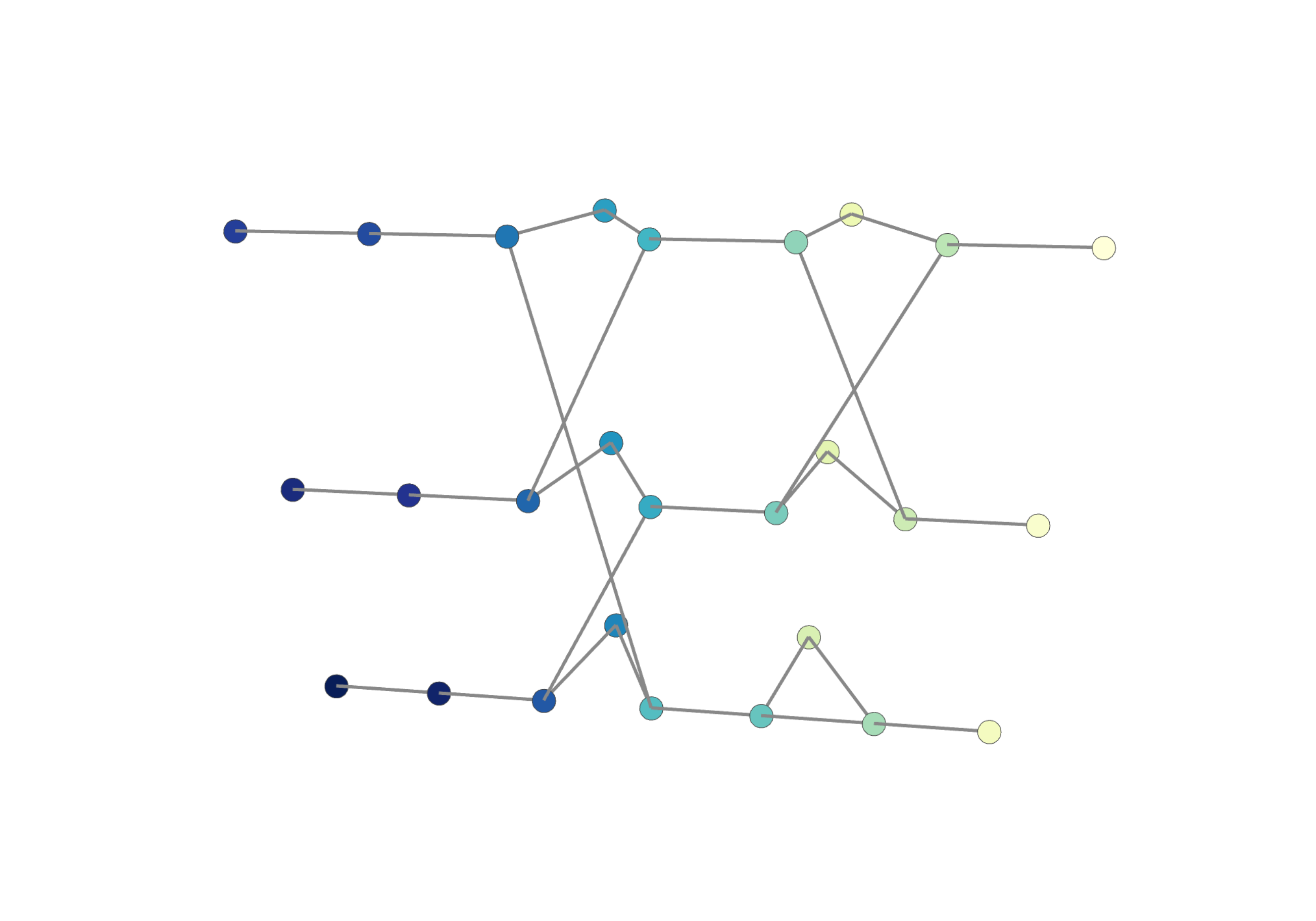}
    \end{subfigure}
    \caption{All degree $3$ covers with base graph from Figure~\ref{fig:basegraph}}
    \label{fig:deg3covers}
\end{figure*}

\begin{figure*}[ht]
    \begin{subfigure}
        \centering
        \includegraphics[width=0.3\textwidth]{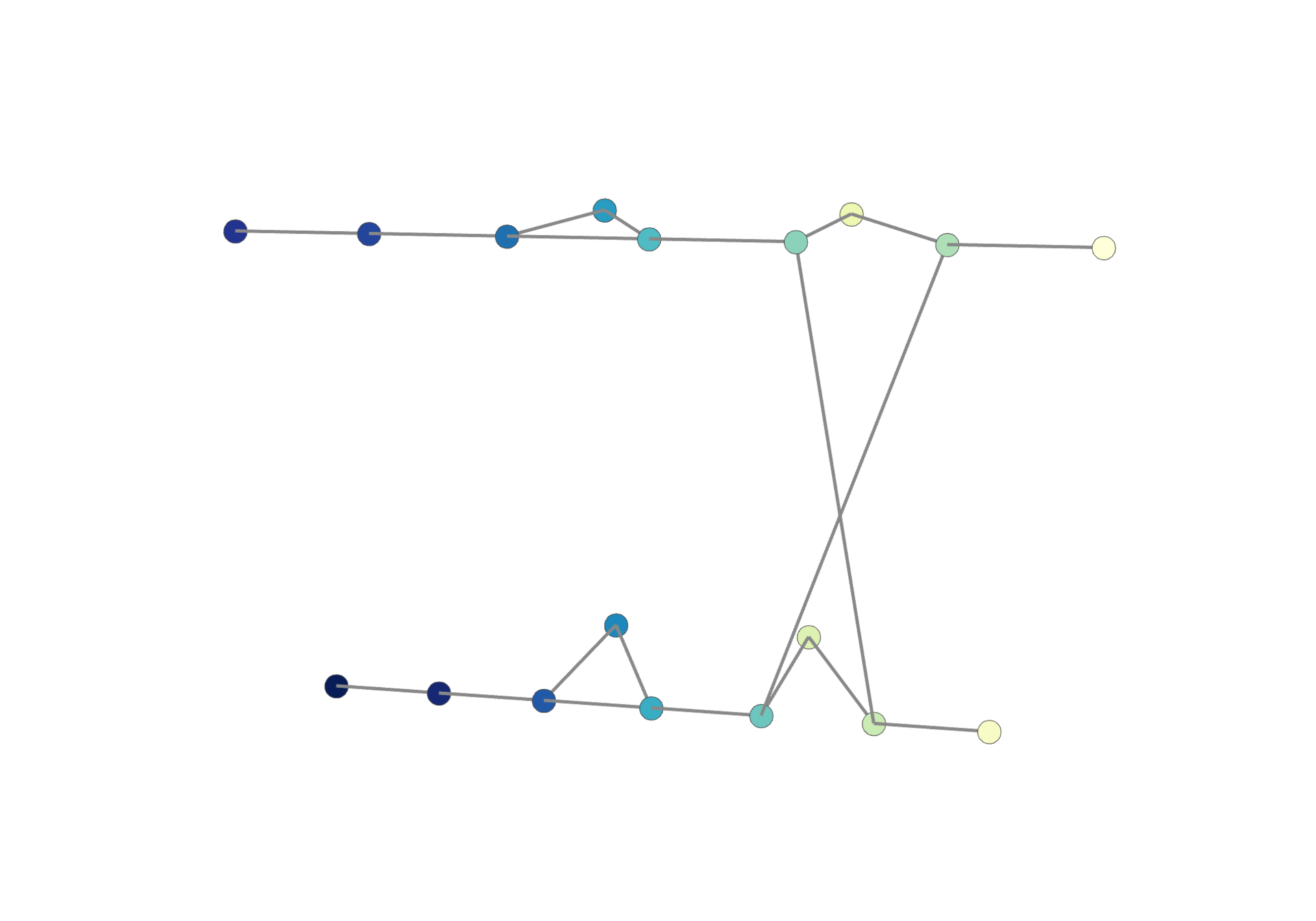}
        \end{subfigure}
        \hfill
    \begin{subfigure}
        \centering
        \includegraphics[width=0.3\textwidth]{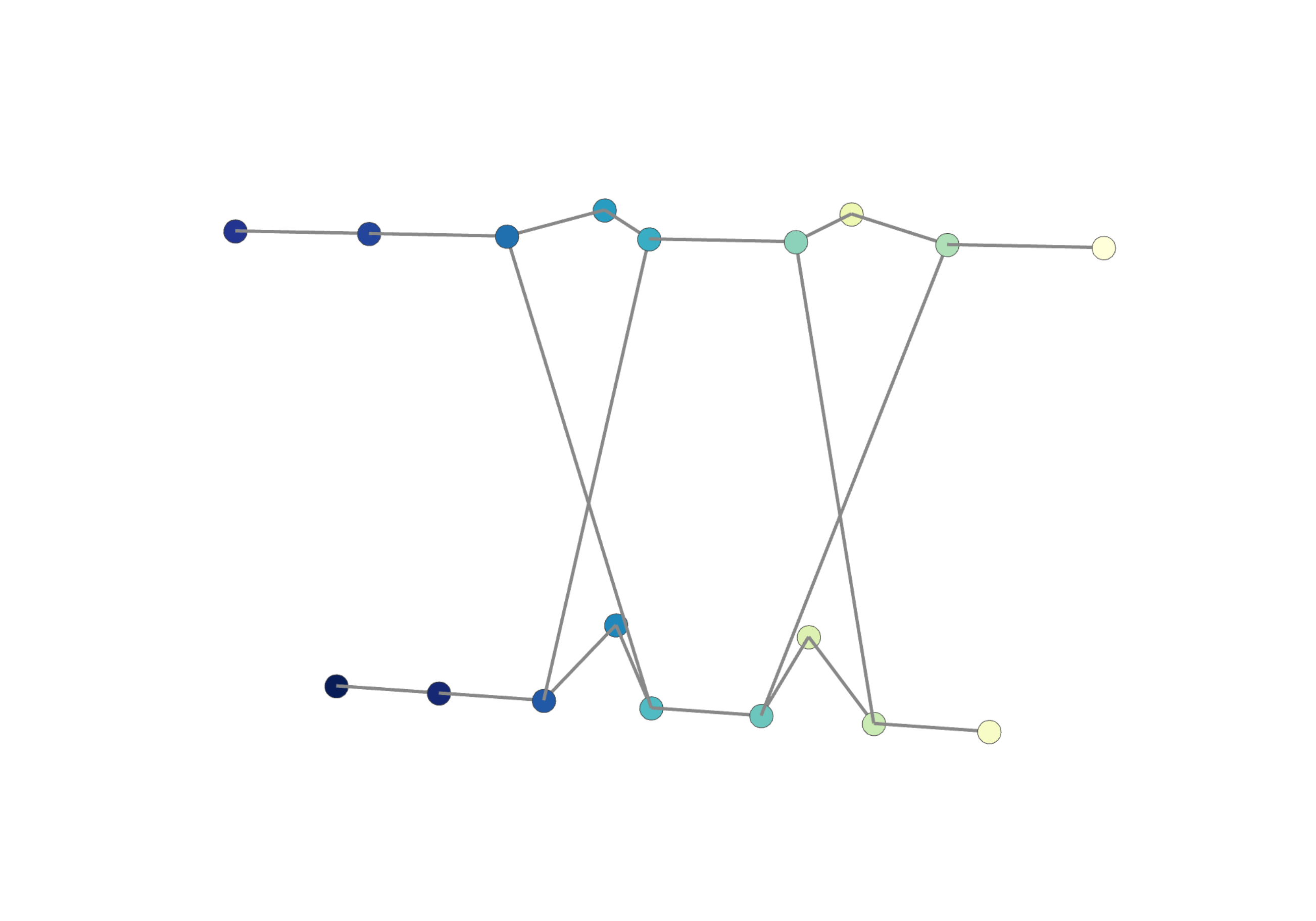}
    \end{subfigure}
    \hfill
    \begin{subfigure} 
        \centering
        \includegraphics[width=0.3\textwidth]{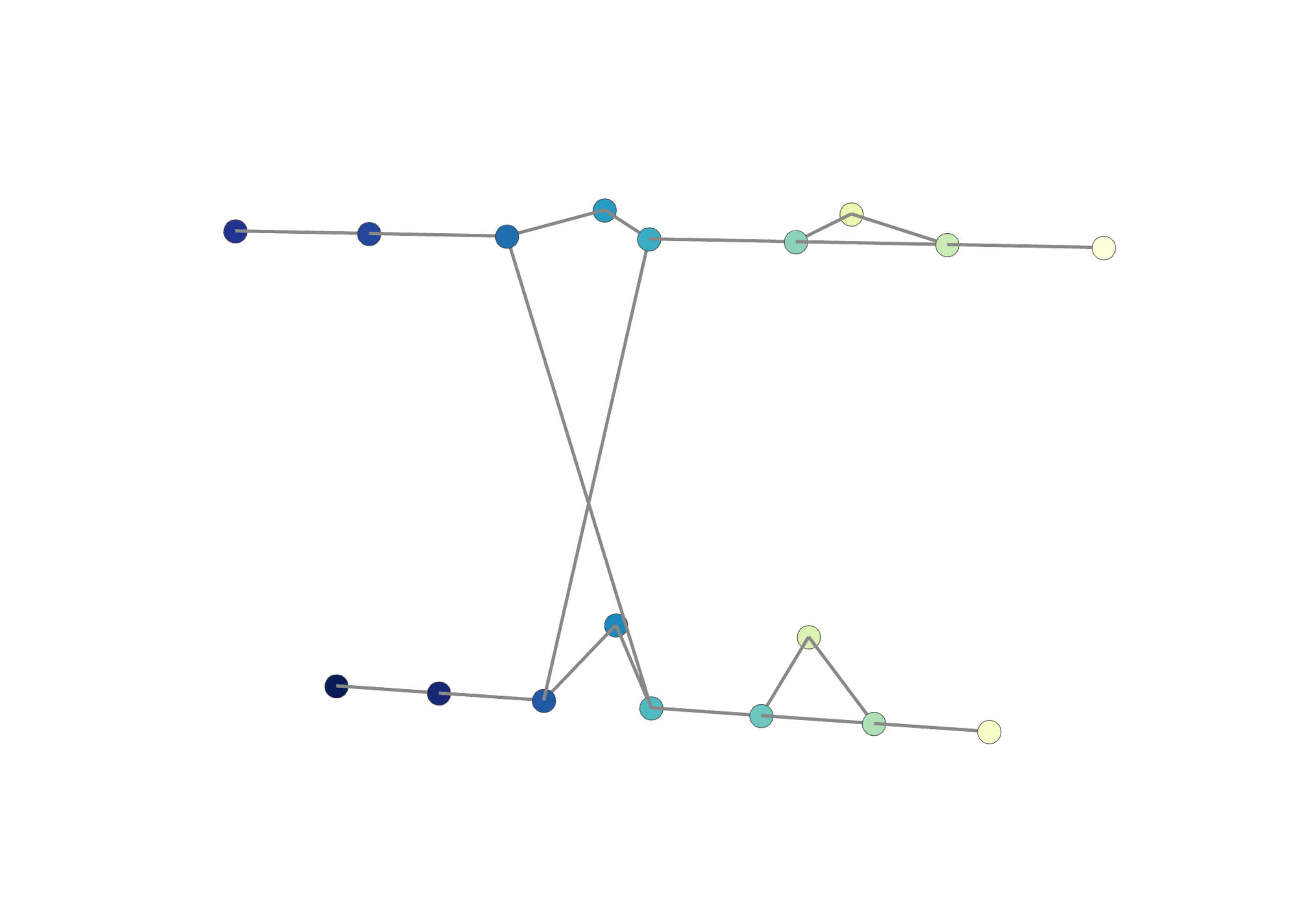}
    \end{subfigure}
    \caption{All degree $2$ covers with base graph from Figure~\ref{fig:basegraph}}
    \label{fig:deg2covers}
\end{figure*}

\begin{figure}[ht]
    \centering
    \includegraphics[width=0.37\textwidth]{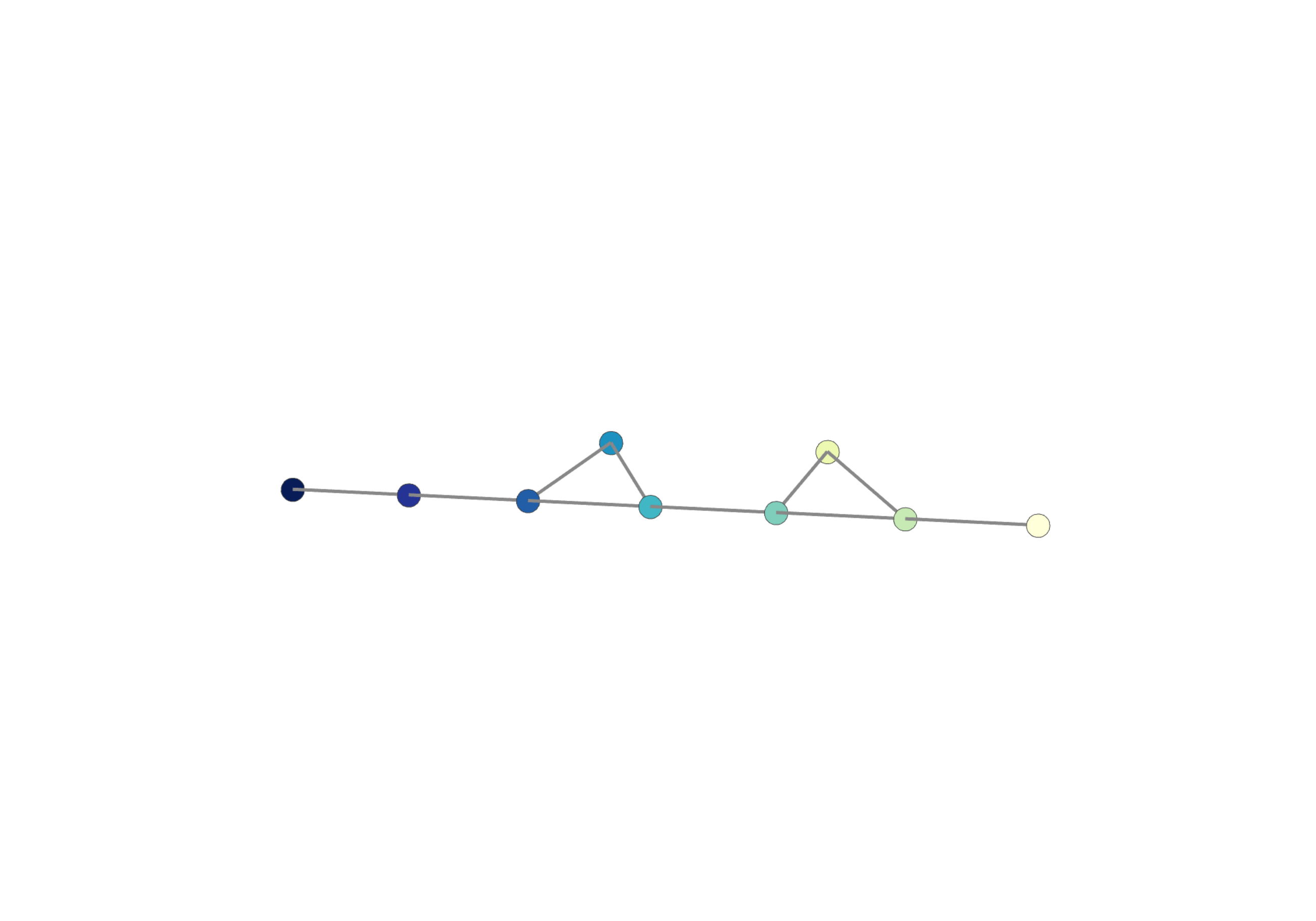}
    \caption{Base graph used in experiments}
    \label{fig:basegraph}
\end{figure}

\end{document}